\documentclass[10pt,twocolumn,letterpaper]{article}

\usepackage[pagenumbers]{cvpr} 

\usepackage[dvipsnames]{xcolor}

\usepackage[hang,flushmargin]{footmisc}

\usepackage{times}
\usepackage{epsfig}
\usepackage{graphicx}
\usepackage{amsmath}
\usepackage{pifont}
\usepackage{amssymb}
\usepackage{lipsum} 

\newcommand{\cmark}{\textcolor{darkgreen}{\ding{51}}}%
\newcommand{\xmark}{\textcolor{red}{\ding{55}}}%

\usepackage{url}           
\usepackage{booktabs}       
\usepackage{nicefrac}    
\usepackage{colortbl}
\usepackage{adjustbox}
\usepackage{caption}
\usepackage{multirow}
\usepackage{bigdelim}
\usepackage{tikz}
\usepackage{tablefootnote}

\usepackage{amsmath,amssymb,amsfonts}
\usepackage{multirow}
\usepackage{array,tabularx,ragged2e}
\newcolumntype{C}[1]{>{\centering\arraybackslash}p{#1}}

\newcommand{\R}{\mathbb{R}}
\newcommand{\La}{\mathcal{L}}
\newcommand{\norm}[1]{\left\lVert#1\right\rVert}
\newcommand{\ra}[1]{\renewcommand{\arraystretch}{#1}}
\usepackage{amsthm}

\newtheorem{lemma}{Lemma}

\newcommand{\C}{\mathbf{C}}
\newcommand{\A}{\mathbf{A}}

\newcommand{\idd}[1]{\mathit{I}_{#1}}

\DeclareMathOperator*{\argmin}{arg\,min}

\newcount\Comments  
\usepackage{color}
\definecolor{darkgreen}{rgb}{0,0.5,0}
\definecolor{purple}{rgb}{1,0,1}

\newcommand\blfootnote[1]{
  \begingroup
  \renewcommand\thefootnote{}\footnote{#1}
  \addtocounter{footnote}{-1}
  \endgroup
}

\definecolor{cvprblue}{rgb}{0.21,0.49,0.74}
\usepackage[pagebackref=true,breaklinks=true,letterpaper=true,colorlinks,bookmarks=false,citecolor=cvprblue]{hyperref}

\def\paperID{171} 
\def\confName{3DV\xspace}
\def\confYear{2024\xspace}

\title{Unsupervised Representation Learning for Diverse Deformable Shape Collections}

\author{
Sara Hahner $^{\star,1,2}$\\
\and
 Souhaib Attaiki $^{\star,3}$\\
\and
 Jochen Garcke $^{1,2}$
\and
 Maks Ovsjanikov $^{3}$
\and
 \normalsize{
 $^{1}$Fraunhofer SCAI,}\\
 \normalsize{Sankt Augustin, Germany}
\and
 \normalsize{
 $^{2}$Institute for Numerical Simulation,}\\
 \normalsize{University of Bonn, Germany
 }
\and
 \normalsize{
 $^{3}$LIX, École Polytechnique,}\\
 \normalsize{Institut Polytechnique de Paris, France}
}

\begin{document}
\maketitle

\blfootnote{
Preprint. Accepted at International Conference on 3D Vision 2024.\\
($^{\star}$) denotes equal contribution
}

\begin{abstract}

We introduce a novel learning-based method for encoding and manipulating 3D surface meshes. Our method is specifically designed to create an interpretable embedding space for deformable shape collections. Unlike previous 3D mesh autoencoders that require meshes to be in a 1-to-1 correspondence, our approach is trained on diverse meshes in an unsupervised manner. Central to our method is a spectral pooling technique that establishes a universal latent space, breaking free from traditional constraints of mesh connectivity and shape categories. The entire process consists of two stages. In the first stage, we employ the functional map paradigm to extract point-to-point (p2p) maps between a collection of shapes in an unsupervised manner. These p2p maps are then utilized to construct a common latent space, which ensures straightforward interpretation and independence from mesh connectivity and shape category. Through extensive experiments, we demonstrate that our method achieves excellent reconstructions and produces more realistic and smoother interpolations than baseline approaches. 

\end{abstract}
\section{Introduction}

Encoding, analyzing, and manipulating 3D surface meshes is a pivotal challenge in 3D computer vision. With the increasing prominence of diverse mesh datasets encompassing humans, animals, and CAD elements, the importance of addressing this issue extends to various applications. These include mesh encoding to low-dimensional latent space \cite{Litany2018}, computer-aided engineering \cite{Hahner2022b}, and mesh generation \cite{Zhou2020}.

Autoencoders have emerged as a potential solution to this challenge. Standard mesh autoencoders, e.g., \cite{Ranjan2018,Bouritsas2019,Yuan2020,Zhou2020}, begin by calculating vertex-wise features. 
They then down-sample the mesh using an encoder to compress the shape representation before reconstructing the original mesh with a decoder.
Alternate strategies, like \cite{Hahner2021,Hahner2022b}, implement autoencoding by initially remeshing input meshes to a semi-regular structure. 
Their autoencoder then handles local patches instead of entire meshes and the added remeshing step often compromises reconstruction quality. 

A significant limitation of autoencoders handling meshes is their requirement for meshes in the shape collections to have a 1-to-1 correspondence, meaning all meshes must utilize the same triangulation—a costly and often impractical demand.
Moreover, accurately determining correspondences across geometric objects is crucial for numerous computer vision and graphic challenges \cite{Bogo2014,Pishchulin2017,Zhou2016,Dinh2005}. 
Various methods have been developed to address this, with the functional map approach \cite{Ovsjanikov2012} showing particular promise. 
Both supervised \cite{donati2020deep,attaiki2021dpfm,Marin2020CorrespondenceLV,marin22_why,Sharp2020} and unsupervised \cite{sharma2020weakly,eisenberger2020deep,halimi2019unsupervised} methods have achieved state-of-the-art results in this area. 
Yet, these shape-matching techniques have not been adapted to mesh autoencoding challenges, necessitating the aforementioned 1-to-1 correspondence.

On the other hand, creating a unified and interpretable embedding space for meshes poses another challenge. 
Techniques that down-sample the input mesh, lead to an embedding space dependent on mesh connectivity. 
Others employ mean or max pooling for vertex features to generate a global shape feature, but this may not always create smooth embedding manifolds.

In our study, we address the above-mentioned issues with a novel mesh autoencoder, trained in an entirely unsupervised manner, that forms a universal latent space unaffected by the shape type or mesh connectivity, enhancing interpretability. For this, we introduce a spectral pooling method to establish this shared space, relying on point-to-point maps between shapes. Advocating for unsupervised methods, we utilize the functional maps pipeline \cite{Ovsjanikov2012} to extract these maps, allowing us to define an embedding space that transcends mesh connectivity and shape categories. The generated shape features reside on a smooth manifold, facilitating interpretable sampling for mesh generation.

\begin{figure*}[!t]
    \centering
    \includegraphics[width=\textwidth]{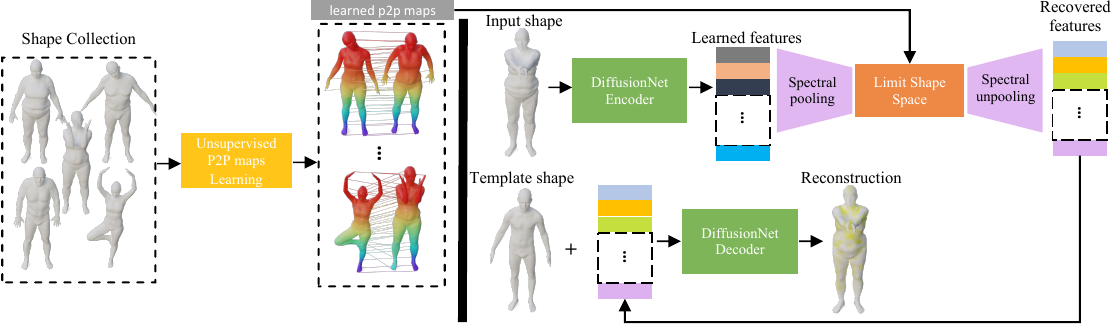}
    \caption{\textbf{Method overview} Our method consists of two stages. In the first stage (left), we train a deep functional map network to extract p2p maps between the input shapes (see Section \ref{sec:stage1}). These p2p maps are then used for spectral pooling by constructing the limit shape space, which is used as the embedding space in our mesh autoencoder in stage 2 (right, see Sections
    \ref{sec:spectral_pooling} and \ref{sec:stage2}).}
    \label{fig:pipeline}
    \vspace{-1.3em}
\end{figure*}

Overall, our primary contributions are:
\begin{itemize}
\item The introduction of a spectral pooling method that disregards mesh connectivity, yielding a shared embedding space for diverse mesh types and categories.
\item A pioneering unsupervised training technique to obtain a mesh autoencoder independent of a fixed mesh template.
\item Demonstrations showcasing our method's capacity to reconstruct superior-quality meshes and generate an interpretable embedding space optimal for shape sampling and manipulation.
\end{itemize}
To facilitate further research and reproducibility, we will provide our code and data upon acceptance.

\section{Related Work}

In this section, we review previous works related to our research.
We organize them into three main categories.

\vspace{-0.3em}

\paragraph{Mesh Autoencoders and Generative Models}

In \cite{Litany2018} and \cite{Ranjan2018} (CoMA), some of the first mesh autoencoders have been introduced.
The authors of CoMA, the Neural3DMM network \cite{Bouritsas2019} and \cite{Yuan2020} utilize mesh downsampling and mesh upsampling layers for pooling and unpooling, which are combined with either spectral or spiral convolutional layers. 
By manually choosing latent vertices for the embedding space, \cite{Zhou2020} defines a MeshConv autoencoder that allows interpolating in the latent space. 
All the above-mentioned mesh convolutional autoencoders work only for collections of meshes with the same connectivity because the pooling and/or convolutional layers depend on the adjacency matrix. 
The authors of \cite{Hahner2021} (MASER) introduce a patch-based approach. 
The meshes have to be remeshed to semi-regular mesh connectivity. The resulting regular patches are input separately to an autoencoder using spatial convolution, allowing for an analysis of meshes of different sizes. 
CoSMA \cite{Hahner2022b} combines this patch-based approach with Chebyshev convolutional filters \cite{Defferrard2016} on the patches.
The MeshCNN architecture~\cite{Hanocka2019} can be implemented as an encoder and decoder.
Nevertheless, the pooling is feature-dependent and therefore, the embeddings can be of different significance.

For surfaces that are represented as signed distance functions and in other implicit representations, \cite{Gropp2020} and \cite{Park2019} achieve good results in shape reconstruction and completion. 
Nevertheless, their generalization and scalability are often limited to a small set of deformations and require big training data.
Another parallel line of work is representation learning on point clouds \cite{Achlioptas2018,Qin2019,Zhao2019}. 
In theory, surface meshes can be handled by these methods when disregarding the faces defining the surface mesh. 
However, these methods only reconstruct and generate point clouds, which is a different and more straightforward task compared to what our work aims for because of their permutation invariance.

The compact representation of the input data by the autoencoder can be used for data generation and manipulation. 
The features are randomly sampled or combined linearly, generating shapes in positions that the user controls.
\cite{Foti2022,Hahner2022b,Ranjan2018} show mesh generative results by sampling from an autoencoder's or variational autoencoder's mesh feature space.
Other generative approaches \cite{Tan2018a,Yang2023} rely on a non-learned deformation representation of meshes of fixed connectivity. 

\vspace{-0.2cm}

\paragraph{Shape Matching} Shape matching has been extensively studied in computer graphics. 
While a comprehensive review is beyond the scope of this paper, interested readers can refer to recent surveys \cite{guo2016comprehensive,bronstein2017geometric,guo2020deep} for a more in-depth discussion. One of the methods most related to our work is the functional map pipeline, which was introduced in \cite{Ovsjanikov2012,Ovsjanikov2017} and has since been extended in many follow-up works \cite{Ren2019,eynard2016coupled,Melzi_2019,attaiki2021dpfm,attaiki2023clover,attaiki2022ncp,attaiki2023vader}. The main advantage of this method is that it transforms the problem of optimizing for a point-to-point map (which is quadratic in the number of vertices) into the optimization of a functional map (which consists of small quadratic matrices), making the optimization process feasible. To find the functional map, earlier works relied on hand-crafted feature functions defined on source and target shapes, such as HKS 
\cite{sun2009concise}, WKS \cite{aubry2011wave}, or SHOT \cite{Salti2014} features. Follow-up research improved the pipeline by introducing additional regularization \cite{Nogneng2017,kovnatsky2013coupled}, 
and proposing efficient refinement methods \cite{Ren2019}. More recently, the functional map pipeline has been incorporated into deep learning, with the seminal work of \cite{litany2017deep} and subsequent works \cite{donati2020deep,sharma2020weakly,attaiki2021dpfm,li2022srfeat} using differentiable functional map losses and regularization to learn feature functions with neural networks. 
Another line of work focused on making the learning unsupervised \cite{roufosse2019unsupervised,eisenberger2020deep,attaiki2022ncp}, which can be useful in the absence of ground truth correspondences. This was achieved by imposing structural properties such as bijectivity and orthonormality on functional maps in the reduced basis \cite{roufosse2019unsupervised}, penalizing the geodesic distortion of the predicted maps \cite{halimi2019unsupervised}, or combining intrinsic and extrinsic shape alignment \cite{eisenberger2020deep}. However, all of these works focused on establishing correspondences and did not investigate any relationship with shape reconstruction or generation.

\vspace{-0.2cm}

\paragraph{Structure of the Feature Space}

The representation learned by an autoencoder typically resides in a lower-dimensional representation space than the input. In this work, our goal is to create a representation space that is shared among different mesh representations and collections. A common method for point clouds is performing (weighted) vertex-wise feature averaging \cite{qi2017pointnet, qi2017pointnet++}.
When neglecting the surface structure defined by the faces, one can apply this approach to the vertices only.
However, this approach is highly sensitive to the distribution of vertices in 3D space, and it cannot guarantee that features of different shapes lie in the same manifold a priori. An alternative approach is to use a template and analyze the features with respect to the template \cite{Ganapathi2018, Garcke2022, Kendall1989}. However, the use of a template can introduce bias. To avoid this, some methods construct a new 3D template shape that resembles the centroid of the collection \cite{Joshi2004}. In our work, we avoid constructing such a shape by using the limit shape basis CCLB \cite{Huang2019}. This approach defines a latent shape in the spectral space to which all shapes are connected via a functional map, thereby avoiding embedding it in the ambient space and introducing potential bias.

\section{Motivation, Notation \& Background}

In this section, we express our motivation for creating an autoencoder that can process meshes with varying connectivities. This is a shift from traditional autoencoders, which mainly work with point clouds and meshes with fixed connectivities. We also touch upon the functional map framework and the idea of limit shape construction. These concepts are crucial to our proposed method. To make the paper easier to follow, we use the same notation throughout.

\subsection{Motivation}

Autoencoders have made significant strides in learning compact representations across various data types. In the 3D domain, they have been particularly successful with point cloud data due to its permutation-invariance property, streamlining the encoding and decoding processes \cite{Achlioptas2018,pang2022masked,wang2021unsupervised,yang2018foldingnet}. However, when applied to triangular meshes, this strength becomes a limitation.

Triangular meshes, in contrast to point clouds, encapsulate the detailed geometry and topology of 3D surfaces, essential for applications like computer graphics where accurate 3D representations underpin realistic renderings. They impose an inherent structure on the 3D data, encoding both geometric and topological relationships among vertices. Unlike the permutation-invariant nature of point clouds, the order in triangular meshes is pivotal as it defines the mesh's connectivity. Tampering with this order could obliterate connectivity data, thereby diminishing the mesh's representational utility. This distinctive characteristic of meshes makes tailoring autoencoders for them notably challenging.

The prevailing approaches to address this challenge often assume that all meshes maintain a 1-to-1 correspondence, meaning they possess identical mesh connectivity \cite{Ranjan2018,Bouritsas2019,Zhou2020}. While this perspective facilitates preserving mesh structures during encoding and decoding through mesh resampling, it also restricts the versatility of these methods. In practice, a strict 1-to-1 correspondence is an exception rather than the rule. Forcing diverse meshes into identical connectivity introduces intricate challenges, often necessitating manual fine-tuning. Such remeshing might produce distortions, undermining the original mesh's quality. Furthermore, when the mesh structure encapsulates salient features about an object, remeshing might not be just unfeasible but also undesirable.

Another ambition in the field is to situate the meshes within a shared embedding space, allowing for both comparative and manipulative operations on the shapes. Contemporary mesh autoencoders, however, hinge on fixed mesh connectivity to form this shared space \cite{Hahner2021,Hahner2022b,Ranjan2018,Bouritsas2019}. 

Motivated by these challenges, our work seeks to develop a novel Mesh Autoencoder (MeshAE) capable of handling arbitrary triangular meshes, thereby eliminating the need for 1-to-1 correspondence, and representing them in a joined embedding space.

\subsection{Notation}
We consider a 3D shape $S_i$, represented as a triangular mesh comprising $n_i$ vertices. We obtain its cotangent Laplace-Beltrami decomposition \cite{laplace} and represent the first $k$ eigenvectors of $S_i$ in the matrix $\Phi_i \in \R^{n_i \times k}$. Additionally, we construct a diagonal matrix $\Delta_i \in \R^{k \times k}$, with its diagonal elements containing the first $k$ eigenvalues of $S_i$. We also define the diagonal matrix of area weights as $M_i \in R^{n \times n}$. It should be noted that $\Phi_i$ is orthogonal with respect to $M_i$ and that $\Phi_i^{\top} M_i \Phi_i = \idd{k}$, where $\idd{k}$ denotes the $\R^{k \times k}$ identity matrix. We further denote $\Phi_i^{\dagger} = \Phi_i^{\top} M_i$ and use the (left) Moore-Penrose pseudo-inverse symbol, $\cdot^{\dagger}$, to represent it.

\subsection{Functional map pipeline}
We use the notation $S_1$ and $S_2$ to refer to a source and target shape, respectively. The pointwise map $T_{12}: S_1 \rightarrow S_2$ is defined as the function that maps each vertex in $S_1$ to a corresponding vertex in $S_2$. To represent this map, we use the matrix $\Pi_{12} \in \R^{n_1 \times n_2}$, which takes the value 1 if $T_{12}(i) = j$, and 0 otherwise. However, with an increasing number of vertices in the shapes, the size of the matrix $\Pi_{12}$ grows quadratically, which is computationally infeasible.

To address this issue, we adopt the functional map paradigm proposed in \cite{Ovsjanikov2012}. This approach reduces the dimensionality of $\Pi_{12}$ by representing it in the spectral basis. Specifically, we construct the functional map $C_{21}$, which maps functions defined on $S_2$ to functions defined on $S_1$, using the expression $C_{21} = \Phi_1^{\dagger} \Pi_{12} \Phi_2$. The functional map has a small size of $(k \times k)$, with $k$ usually around 30, making the optimization process feasible.

To find the functional maps that map $S_1$ and $S_2$, we first obtain two $d$-dimensional feature functions, also known as probes, $F_1$ and $F_2$ defined on $S_1$ and $S_2$ respectively ($F_i \in \R^{n_i \times d}$). We then compute the coefficients $\A_i$ of the feature functions in their corresponding reduced basis using $\A_i = \Phi_i^{\dagger} F_i$. Next, we formulate an optimization problem: 
\begin{align}
\argmin_{\C} \| \C \A_1 - \A_2 \|_F^2,
\label{eq:fmap_basic}
\end{align}
where $\C$ is the sought-after functional map. 

\subsection{Canonical Consistent Latent Basis} 
\label{sec:FMN}

Given a collection of related  3D shapes $S_1, \dots, S_n$, and a set of functional maps between some shape pairs, we build a functional map network on the collection as follows. We construct a graph $\mathcal{G} = (\mathcal{V}, \mathcal{E})$, where the $i$-th vertex represents the functional space of the shape $S_i$, and the edge $(i,j)$ exists if the functional maps $C_{ij}$ and $C_{ji}$ are given, in which case, the graph is symmetric. We assume that our graph is connected, which means that there exists a path between any two shapes in the collection.

With this construction in hand, we can translate functions between any shapes $S_i$ and $S_j$ in the shape collection. Nevertheless, we do not have a common basis. We solve this by using the limit shape construction as in \cite{Wang2013}, which provides a  latent basis $Y_i$ for the collection's shape features, such that $C_{i,j} Y_i \approx Y_j, \forall i,j$. 
These latent bases $Y_i \in \R^{k_1 \times k_1}$ ($k_1$ is the same dimension of the functional maps) can be interpreted as functional maps from a \textit{latent shape} to each shape $S_i$.

To further enhance the stability of this construction and eliminate shape metric ambiguity, \cite{Huang2019} introduced the canonical consistent latent basis (CCLB) $\widetilde{Y_i} \in \R^{k_1 \times k_2}$, which has been shown to yield better results. The CCLB enables unbiased comparisons of the shape features in the collection. Therefore, we use this common basis to define the embedding space of our autoencoder, which captures the diversity of our shape collection.

\section{Method}

In this section, we introduce our proposed model for shape representation and generation, resolving the challenges motivated in the previous section. 
For that, we introduce a novel spectral mesh pooling and present an unsupervised learning method of functional maps to construct point-to-point maps between a collection of shapes, 
This is the first stage of our approach, for which we provide an overview in Figure \ref{fig:pipeline}. 
The second stage of our model is an autoencoder making use of the novel spectral pooling.
We will publish our complete code and data.

\begin{figure*}[t]
\begin{tabular}{C{0.1\linewidth} C{0.12\linewidth} C{0.13\linewidth} C{0.12\linewidth} C{0.12\linewidth} C{0.12\linewidth} C{0.1\linewidth}}   
GT & CoMA  & Neural3DMM & MeshConv & MASER & CoSMA & Ours\\
\end{tabular}
{\includegraphics[width=\linewidth, trim=0cm 10cm 0cm 0.5cm, clip]{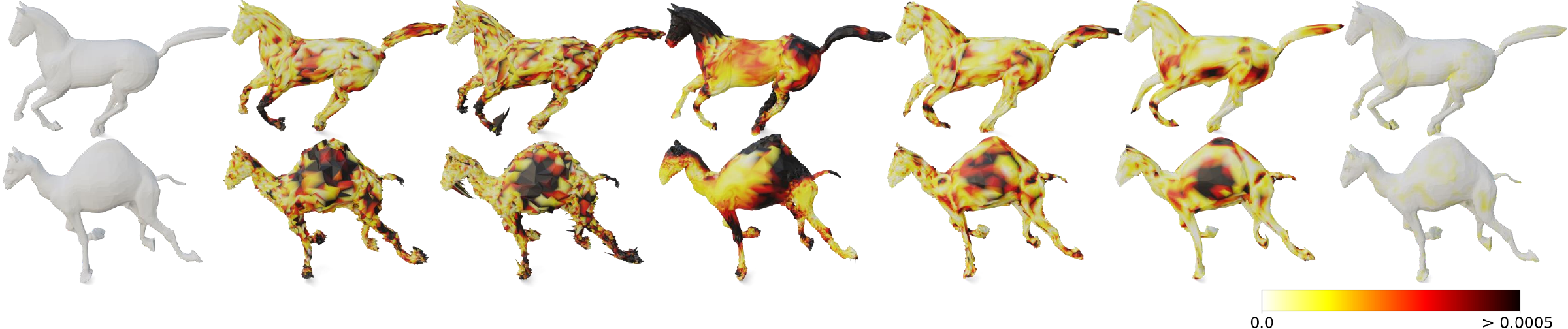}} 
\caption{Reconstructed meshes from the \textit{GALLOP} dataset. Vertex-wise MSE is highlighted.}
\label{recon_gallop}
\vspace{-1.2em}
\end{figure*}

\subsection{Spectral Pooling}
\label{sec:spectral_pooling}
We develop a spectral mesh pooling operator to reduce the dimensionality of the meshes in the spectral domain to handle meshes of different connectivity and represent them in a joined low-dimensional embedding space. 

In the case of classical representation learning for 2D images with convolutional networks, all image samples have a fixed size and are in 1-to-1 correspondence. 
The convolutional filters calculate vertex-wise features, then pooling summarizes many vertex-wise features, reducing the number of pixels.  
This is done uniformly for all the images in correspondence, and hence, features from different samples are comparable to each other. 
Therefore, pooling in 2D can also be interpreted as a projection from a high dimensional basis to lower dimensional basis functions. Here, the cardinality of the basis is equal to the number of pixels. 
Because of the 1-to-1 correspondence, all the images are described in the same basis. 
A similar pooling operator cannot be constructed for meshes with different mesh connectivities. 
We can only obtain point-to-point maps between the shapes that allow the projection of a function from one shape to another.

To solve the pooling for meshes, we propose to adapt the CCLB method (initially developed for deformation detection) and introduce a novel intrinsic spectral mesh pooling.
We project vertex-wise features that are calculated for every shape separately to the common CCLB basis, reducing the dimension from the number of vertices to the size of the limit shape. 
We calculate the limit shape basis CCLB as described in \ref{sec:FMN}. 
It has dimension $k_2$ and uses eigendecompositions of the Laplacians of size $k_1 \geq k_2$. It will be the common basis for the low-dimensional embedding space.
For the spectral unpooling, we project the features from the limit shape basis back to the vertex representation.

Note that in the special case, when the dimension of the limit shape equals one ($k_1=k_2=1$), the spectral pooling corresponds to a global $\pm$ mean pooling for all the shapes in the collection. 
The spectral unpooling duplicates the average feature into the vertices of the shape.
We formally state and prove this observation in the supplementary materials.

\subsection{Unsupervised Maps Extraction}
\label{sec:stage1}
This step aims to generate point-to-point (p2p) maps between a collection of shapes for training the autoencoder. While most mesh autoencoder requires p2p supervision, obtaining p2p maps is challenging since it requires significant labeling effort, which is prohibitive. To overcome this challenge, we propose learning approximate p2p maps in an unsupervised manner and introducing additional regularization in the loss to rectify the defaults in the maps.

To achieve this, we choose to learn the maps using the unsupervised functional maps setting. We followed the approach of \cite{donati2020deep} by training a DiffusionNet network \cite{Sharp2020} to generate feature functions that will estimate a functional map between a source and a target shape. We supervise the training by imposing structural properties on the functional maps. 

Specifically, given a source and target shape $S_1$ and $S_2$, we first extract $d$-dimensional feature functions $F_1$ and $F_2$, respectively using DiffusionNet. We then project these features onto the reduced Laplacian basis $\A_i = \Phi_i^{\dagger} F_i$. Next, we estimate the functional map between $S_1$ and $S_2$ using:
\begin{align}
\argmin_{\C} \| \C \A_1 - \A_2 \|_F^2 + \lambda \|\C \Delta_1 - \Delta_2 \C\|_F^2.
\label{eq:fmap_network}
\end{align}

The second term is a regularization that promotes the isometry of the maps, as described in \cite{Ovsjanikov2012}. This operation is differentiable. To train the network, we predict the functional map in both directions (\ie, $\C_{12}$ and $\C_{21}$) and then penalize the deviation of the predicted maps from bijectivity and orthogonality. 
The first loss requires the maps to be the inverse of each other, while the second loss regularizes the maps to be locally area-preserving, as several previous works demonstrate \cite{sharma2020weakly,roufosse2019unsupervised}. We can write these losses as:
\begin{align}
L = \| \C_{12}\C_{21} - I \|_F^2 + \sum_{i,j \in \{1, 2\}} \| \C_{ij}^{\top}\C_{ij} - I \|_F^2
\label{eq:fmap_loss}
\end{align}

Once the network is trained, we extract functional maps between all pairs of shapes and convert them to p2p maps. To improve the quality of the maps, we use the recent refinement method ZoomOut \cite{Melzi_2019}. This method navigates between the spectral and spatial domains while progressively increasing the number of spectral basis functions. The final maps are then used to train the autoencoder.

\subsection{Our Architecture}
\label{sec:stage2}

Given a shape collection of meshes that can have different connectivity, we define our architecture employing the contributions explained in the previous paragraphs.
Using existing (ground truth) or unsupervised learned point-to-point maps (as in Section \ref{sec:stage1}), we calculate functional maps between the shapes and then construct the functional map network, as well as the limit shape basis CCLB for the introduced spectral pooling.
In addition, we chose a set of template meshes from the collection for the different categories of meshes, which will be used for the reconstructions.

Our autoencoder makes use of the surface-based convolutional network DiffusionNet \cite{Sharp2020}, which has proven to learn discretization agnostic vertex-wise shape features. 
We input the vertex 3D coordinates of shape $S_i$ to the \textbf{encoder}. Four trainable DiffusionNet Blocks \cite{Sharp2020} are applied to calculate $F$-dimensional vertex-wise features. Then we apply spectral pooling, and these features are projected to the CCLB by multiplying them from the left by $Y_i^{\dagger} \Phi_i^{\dagger}$. This low-dimensional representation $z_i$ of dimensionality $F \cdot k_2$ is now independent of the mesh connectivity of $S_i$ because it is represented in the common CCLB basis. 

The \textbf{decoder} applies spectral unpooling and projects the features represented in the CCLB to the template shape $S_t$ by multiplying it by $\Phi_t Y_t$ from the left. 
At this point, we concatenate the vertex-wise 3D coordinates of the template shape to the projected features to provide more information for the reconstruction of the input shape. 
Finally, four trainable DiffusionNet Blocks reconstruct the 3D coordinates of the input shape on the template mesh's vertices.

\subsection{Losses}
Our autoencoder is fully differentiable, and we denote the input shape as $S$. The encoder and the decoder are respectively represented as $enc$ and $dec$, the reconstruction is  $X = dec(enc(S))$. We train our network using two losses.

\textbf{Point-to-point (p2p) loss}: Given a point-to-point map $\Pi$ (either ground truth or extracted by the first stage) between the template and the input shape, the p2p loss is defined as $L_1 = \|\Pi S - X\|_F^2$. However, in the case of unsupervised maps, this loss may provide inaccurate signals as the p2p map is often faulty and not entirely correct. To address this issue, we use an additional loss.

\textbf{Reconstruction loss}: Given the reconstruction $X$, we construct the matrix $D^X$ such that $D^X_{i,j} = \|X_i - X_j\|_F^2$. We create the matrix $D^S$ for $\Pi S$ in the same manner. The reconstruction loss is $L_2 = \|D^S - D^X\|_F^2$. This loss computes the cumulative reconstruction error and each point receives reconstruction feedback from the other $n-1$ points. Thus, even if the p2p map is faulty in some places, the faulty points receive signals from the non-faulty ones. As this loss is rotation invariant, it cannot be used alone. Our final loss combines the two losses: $L = L_1 + \lambda L_2$.

\section{Experiments}

In this section, we evaluate our architecture on various tasks using three different shape collections.





\begin{table*}[t]  
    \centering
    \ra{1.0}
        \resizebox{\linewidth}{!}{
            \begin{tabular}{@{}lccrrrrrc@{}}

\toprule
& & & \multicolumn{2}{c}{\textit{FAUST}} & \multicolumn{3}{c}{\textit{GALLOP}} & \textit{TRUCK} \\
Method  & Unsupervised & No remesh & Unknown poses & Unknown indiv. & Camel & Elephant & Horse & ($\times$100) \\
\cmidrule(r{.5em}){1-3}\cmidrule(lr{.5em}){4-5}\cmidrule(l{.5em}){6-8} \cmidrule(l{.5em}){9-9}
CoMA & \xmark & \cmark & 
    569.3 $\pm$ 203.1 & 28.3 $\pm$ 6.4 & 7.8 $\pm$ 1.4 & 24.3 $\pm$ 4.4 & 3.2 $\pm$ 0.3 & - \\
Neural3DMM & \xmark & \cmark &
    246.2 $\pm$ 5.4 & 10.4 $\pm$ 0.9 & 12.4 $\pm$ 0.1 & 29.7 $\pm$ 3.5 & 4.7 $\pm$ 0.1 & - \\
MeshConv & \xmark & \cmark &
    18.2	$\pm$	2.2 & 3.5	$\pm$	0.4 & 9.2	$\pm$	0.4 & - \textsuperscript{\ref{meshconv_ele}} & 7.3	$\pm$ 0.1 & - \\
MASER & \xmark & \xmark & 
2.8	$\pm$	0.2  & 1.3	$\pm$	0.1 & 4.2	$\pm$	0.03 & 17.8	$\pm$	0.8 & 1.8	$\pm$	0.02 & 187.8 $\pm$ 11.7\\
CoSMA & \xmark & \xmark & 
\textbf{1.1	$\pm$	0.01}  & 0.9	$\pm$ 0.02 & 3.3	$\pm$	0.01 & 20.0	$\pm$ 0.3 & 1.2	$\pm$	0.01 & 15.7 $\pm$ 0.5\\
Ours - supervised & \xmark & \cmark &
    2.4 $\pm$ 0.08 & \textbf{0.7 $\pm$ 0.01} & \textbf{0.3 $\pm$ 0.03} & \textbf{1.09 $\pm$ 0.08} & \textbf{0.11 $\pm$ 0.01} & \textbf{1.01	$\pm$	0.1}\\
\addlinespace
Global pooling & \cmark & \cmark & 475 $\pm$  26.2 & 35.5 $\pm$ 1.0 & 25.2 $\pm$ 5.6 & 22.4 $\pm$ 0.2 & 3.9 $\pm$ 0.9 & 14.6 $\pm$	1.9\\ 
Ours - unsupervised & \cmark & \cmark & \textbf{4.3 $\pm$ 0.1} & \textbf{2.0 $\pm$ 0.05} & \textbf{6.8 $\pm$ 0.3} & \textbf{21.0 $\pm$ 0.2} & \textbf{1.2 $\pm$ 0.02} & \textbf{10.9	$\pm$	0.6} \\
\bottomrule

            \end{tabular}
        }

\caption{MSE errors between the reconstructed and original mesh of the \textit{FAUST}, \textit{GALLOP}, and \textit{TRUCK} datasets. The reported numbers are mean errors over 3 runs randomly initialized. $\pm$ denotes the standard deviation.}
\label{tab:mse_faust}
\vspace{-1.3em}
\end{table*}

\subsection{Shape Collections}

We conduct experiments using three distinct datasets previously utilized in recent studies \cite{Hahner2021,Hahner2022b}. 

The \textit{GALLOP} shape collection contains triangular meshes representing a motion sequence with 48 timesteps from a galloping horse, elephant, and camel \cite{Sumner2004}. The galloping movement is similar but the meshes representing the surfaces of the three animals differ in connectivity and the number of vertices. We use the last 14 timesteps for testing.

The \textit{FAUST} collection contains 100 meshes \cite{Bogo2014}. 
The irregular surface meshes represent 10 different bodies in 10 different poses. 
We apply two different train-test splits, following previous works \cite{Hahner2022b}. 
In the first setting, known as ``unknown poses'', the network is trained on 8 poses out of 10, and tested on the remaining 2, while in the second setting, known as ``unknown individuals'', the network is trained on 8 individuals and tested on the remaining 2.

The  \textit{TRUCK} shape collection \cite{Hahner2022b} contains 32 completed frontal car crash simulations of 6  different components \cite{NCAC}. 
Only 10 simulations are included in the training set. 
In this dataset, the components represented by surface meshes often deform in different patterns during the crash. 
One goal is to detect clusters corresponding to different deformation patterns in the components’ embeddings in order to speed up the analysis of car crash simulations \cite{Bohn2013}.

\subsection{Results}

\begin{figure}[t]
\small
\begin{tabular}{C{0.07\linewidth} c @{\hspace{1em}} c @{\hspace{1em}} c @{\hspace{0.8em}} c @{\hspace{1em}} c C{0.05\linewidth}}
\footnotesize
  \multirow{2}{*}{GT}  & \multirow{2}{*}{CoMA}  & Neural & Mesh &
  \multirow{2}{*}{MASER} & \multirow{2}{*}{CoSMA} & \multirow{2}{*}{Ours}\\
   & & 3DMM & Conv\\
\end{tabular}
   {\includegraphics[width=\linewidth, trim=0cm 8.5cm 13.9cm 0cm, clip]{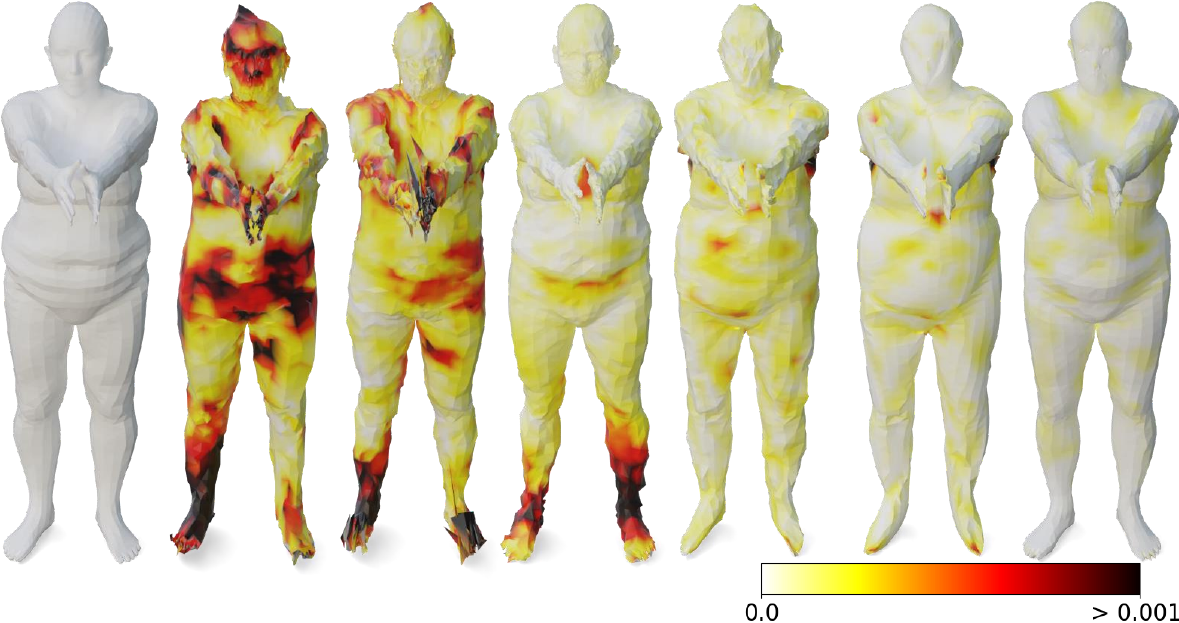}} 
\caption{Reconstructed meshes from the \textit{FAUST} dataset of the "unknown individuals" setup. Vertex-wise MSE is highlighted.}
\label{recon_faust}
\vspace{-1em}
\end{figure}

\begin{figure}[t]
\centering
   {\includegraphics[width=0.9\linewidth, trim=0cm 8.5cm 14cm 0cm, clip]{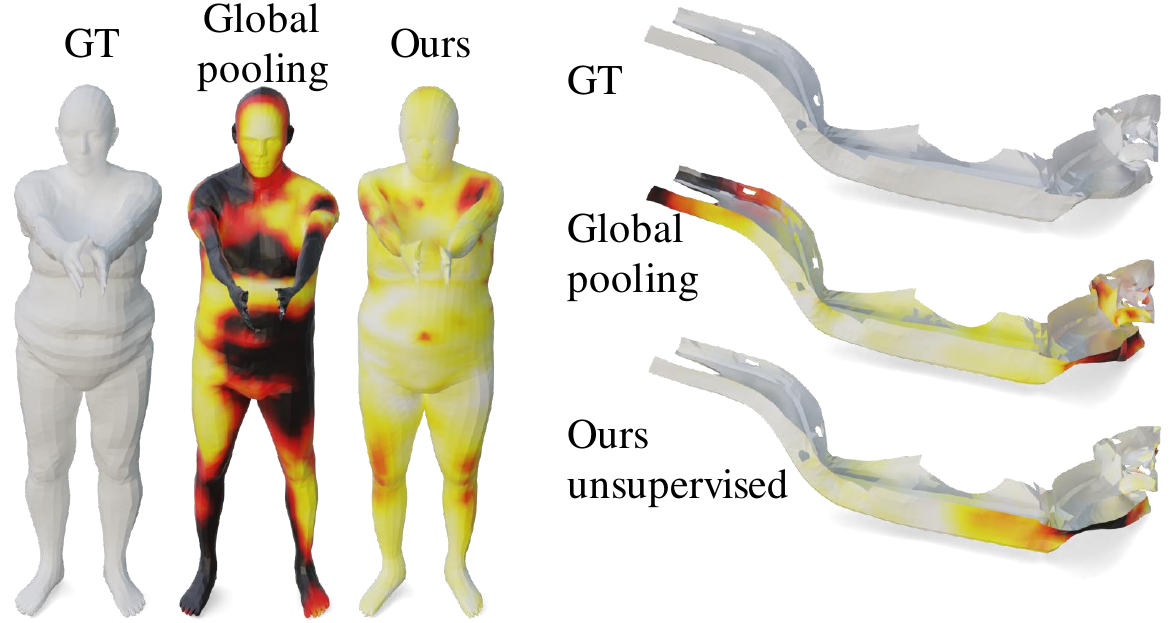}} 
\caption{Reconstructions from the unsupervised experiments on the \textit{TRUCK} and \textit{FAUST} datasets. Vertex-wise MSE is highlighted using the same color range as for the supervised experiments.}
\label{recon_unsup}
\vspace{-1.2em}
\end{figure}

We compare our method to five recent baseline architectures: CoMA \cite{Ranjan2018}, Neural3DMM \cite{Bouritsas2019}, MeshConv \cite{Zhou2020}, MASER \cite{Hahner2021}, and CoSMA \cite{Hahner2022b}. 
For all baseline autoencoders, we chose embedding sizes following previous works. 
The first three do not allow an analysis of meshes with different mesh connectivity by the same trained architecture. 
The latter two methods allow an analysis of different meshes with different connectivity after being remeshed to a semi-regular mesh representation by inputting patches of the meshes to the AE. 
Nevertheless, their shape features depend on the semi-regular mesh connectivity; hence, the embedding space is not joint.
The reconstructed semi-regular meshes are projected back to the original meshes using a parametrization to calculate the error.
All these baseline mesh AE are supervised, so we compare our approach to them using supervised point-to-point maps. 

The second is to train the autoencoder using the unsupervised maps produced by the first stage, see section \ref{sec:stage1}.
As a comparison, we construct a baseline method that uses unsupervised point-to-point maps and global average pooling instead of the introduced spectral mesh pooling. 
This corresponds to the case when the dimensionality of the CCLB is 1 ($k_1=k_2=1$), see section \ref{sec:spectral_pooling}.

\begin{figure*}[t]
\centering
   {\includegraphics[width=0.37\linewidth]{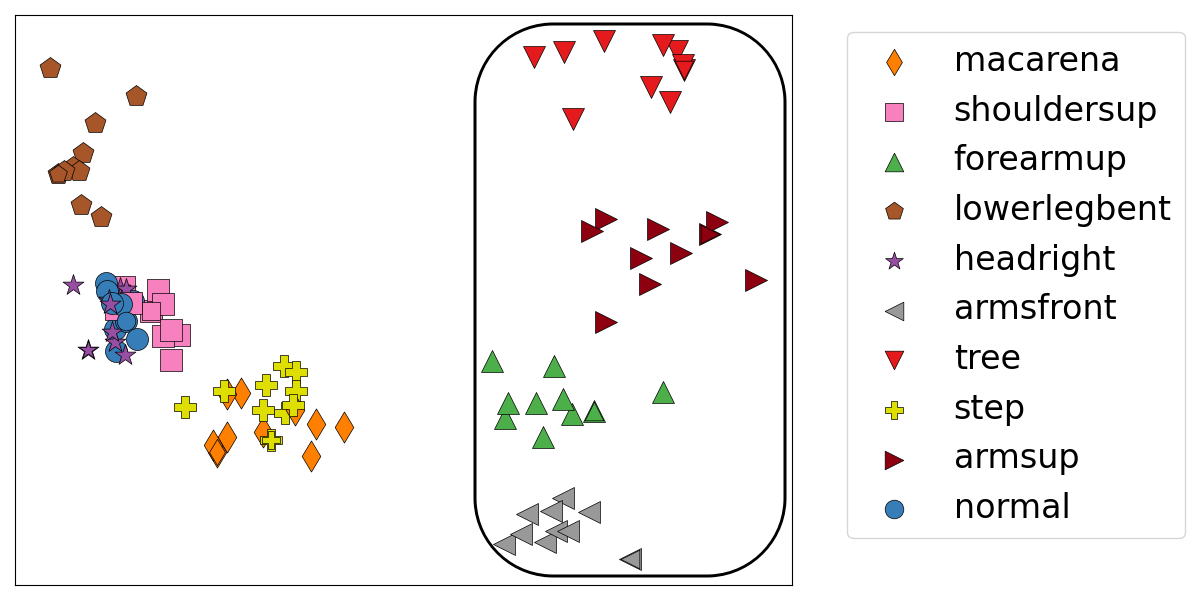}}
   \hspace{1em}
   {\includegraphics[width=0.25\linewidth]{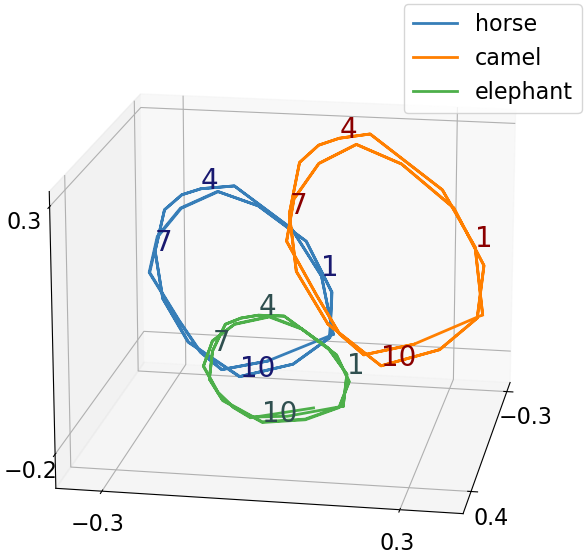}} 
   \hspace{1em}
   {\includegraphics[width=0.28\linewidth]{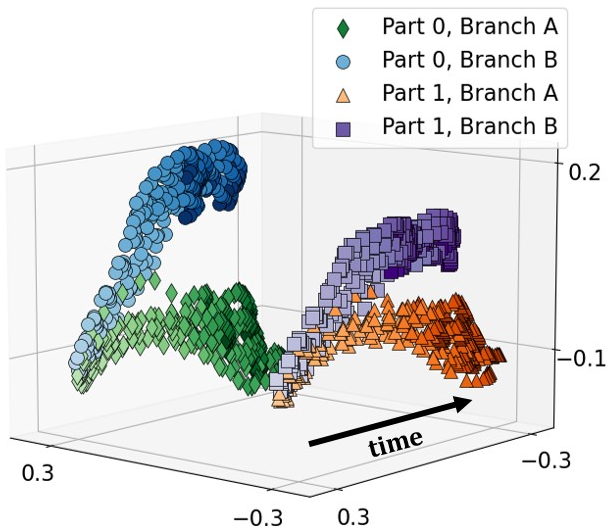}}
\caption{Embeddings in 3D or 2D of the learned representations in the common basis. 
Left: \textit{FAUST} positions marked with a triangle raise the arms.
Middle: galloping sequences from the \textit{GALLOP} dataset with timesteps provided in the plot. 
Right: two \textit{TRUCK} components deform in two clusters over time, corresponding to the deformation patterns (Branch A and B) visualized in Figure \ref{fig:recon_truck}.   
}
\label{emb_gallop}
\vspace{-1em}
\end{figure*}

\footnotetext[1]{\label{meshconv_ele}MeshConv AE for the elephant is too large to train on 40 GB GPU.}

\subsubsection{Mesh Reconstructions}

We initiate our analysis by conducting a conventional reconstruction experiment. 
First, we encode a shape $S$ from the test set, which was never seen during the training phase, into a latent code. 
This is decoded subsequently using our decoder.
We compare the output to the initial shape to assess the reconstruction. 
We sum up the vertex-wise mean squared errors (MSE) between the vertex coordinates of the input shapes and their reconstructions to determine the reconstruction error. 
To obtain uniform results, we normalize all meshes into the range $[-1,1]$. 
We report all reconstruction errors in Table \ref{tab:mse_faust}.

For the \textit{FAUST} dataset, our supervised method achieves the best result in the "unknown individuals" setting and the second-best result in the "unknown poses" setting (see Figure \ref{recon_faust} for a visualization). 
In addition, our results are more stable than some of the baselines, as indicated by the standard deviation. 
Our unsupervised results are better than the supervised results that do not require any remeshing, which demonstrates the usefulness of our approach and the regularization introduced by the losses to mitigate errors in the maps. 
Additionally, our spectral pooling strongly improves the reconstruction quality for the unsupervised experiments compared to using global pooling in the encoder, see Figure \ref{recon_unsup}.

For the \textit{GALLOP} dataset, we train our network on all categories in the supervised setting. 
However, due to the highly non-isometric nature of the three categories, most unsupervised methods for shape matching fail. Thus, we train our unsupervised method on each category individually. 
The mesh-dependent baselines are also trained on each animal separately. Only MASER and CoSMA train on the three animals together since mesh patches are input separately. 
Our supervised method achieves the best results for all categories, see Table \ref{tab:mse_faust}. 
Reconstructed meshes are visualized in Figure \ref{recon_gallop}. 
Concerning our unsupervised method, it achieves comparable results with the baselines and outperforms the unsupervised global pooling approach. 
This demonstrates that learning high-quality mesh autoencoders is possible even in the absence of ground truth maps.

Finally, we report in Table \ref{tab:mse_faust} the result on the \textit{TRUCK} dataset. Due to its big size, we only test our method against the best two performing methods. 
Once again, our method achieves the best results in the supervised case.
Additionally, our unsupervised reconstruction quality is superior to all supervised baselines.
We provide visualizations of the reconstructed meshes in the supplementary material for the supervised and Figure \ref{recon_unsup} for the unsupervised methods.

Qualitatively, our reconstructed meshes are smooth, deform naturally, and do not have any outlier vertices, which is not the case for some baseline methods. 
The provided reconstructed meshes from all three datasets and supervised and unsupervised experiments in Figures \ref{recon_gallop} to \ref{recon_unsup} are smooth and have the lowest reconstruction error.

\subsubsection{Low-Dimensional Embeddings}

For every mesh from the collections, we obtain a hidden representation of size $k_2 \times F$. The shape features from the same collection can be visualized in 2D or 3D using a principal component analysis \cite{Pearson01}, see Figure \ref{emb_gallop}.

Similar to the other approaches, we embed the different shape categories separately from each other. 
In the case of the \textit{FAUST} dataset, several clusters form in the embedding space of the unsupervised experiment, which corresponds to different positions. Additionally, along the horizontal axis, the position of the arms can be split into raised or not raised.

Additionally, for the \textit{first time}, we can jointly visualize the features from various shapes of different connectivity in a common basis. 
It allows for a joint visualization of the galloping sequences of camel, horse, and elephant from the \textit{GALLOP} shape collection. 
The MASER and CoSMA baselines, on the other hand, only generate embeddings of every animal separately.
Figure \ref{emb_gallop} visualizes the learned features in 3D for the supervised experiment because the unsupervised one was conducted on the animals separately.
The sequences align over time up to translation but are still separated from each other, which captures the different shape categories.  

For two \textit{TRUCK} components, we aim to detect two clusters corresponding to a different deformation behavior, similarly to \cite{Hahner2022b}. 
These different \textit{TRUCK} components can, for the \textit{first time}, be visualized together using the representation in the CCLB. 
The two deformation branches in two different components are split along the same axis of the 3-dimensional embedding space, and the features of both components align over time, see Figure \ref{emb_gallop} for the embedding from the unsupervised experiment. 
This visualizes nicely that the deformation of the two components manifests in similar deformation patterns.

\begin{figure}[t]
\vspace{-1em}
\begin{center}
    {\includegraphics[width=0.9\linewidth, trim=8.7cm 4.8cm 4.8cm 4cm, clip]{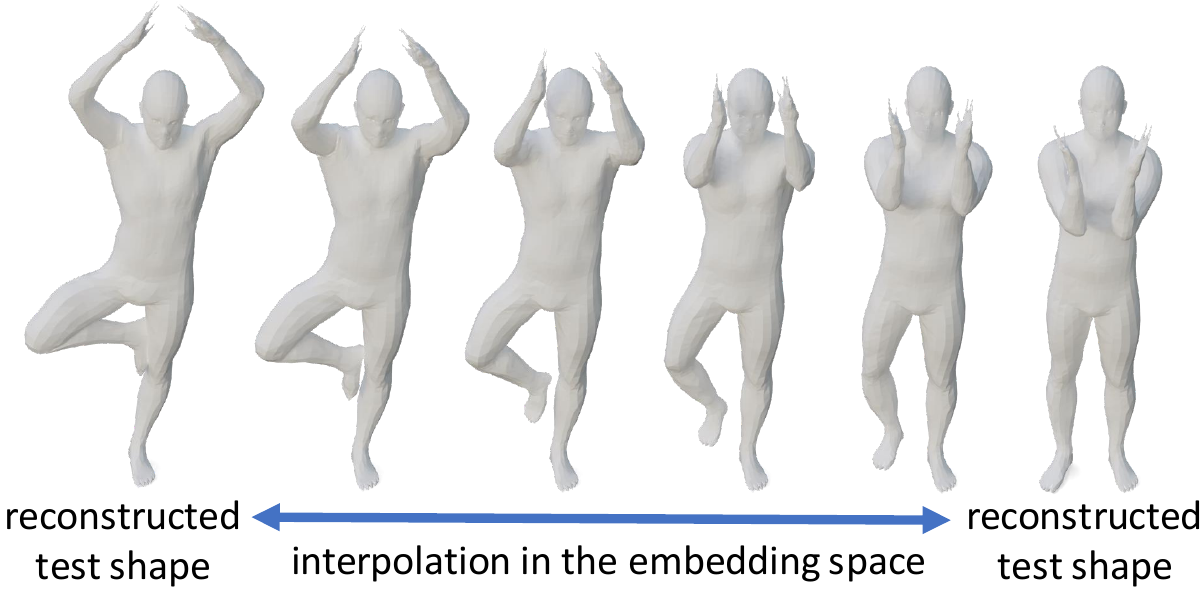}}
    \caption{Interpolating between different \textit{FAUST} test shapes.}
    \label{gen_faust}
\end{center} 
\vspace{-2em}
\end{figure}

\begin{figure}[t]
\vspace{-1em}
\begin{center}
    {\includegraphics[width=0.88\linewidth, trim=0cm 0cm 8.5cm 0cm, clip]{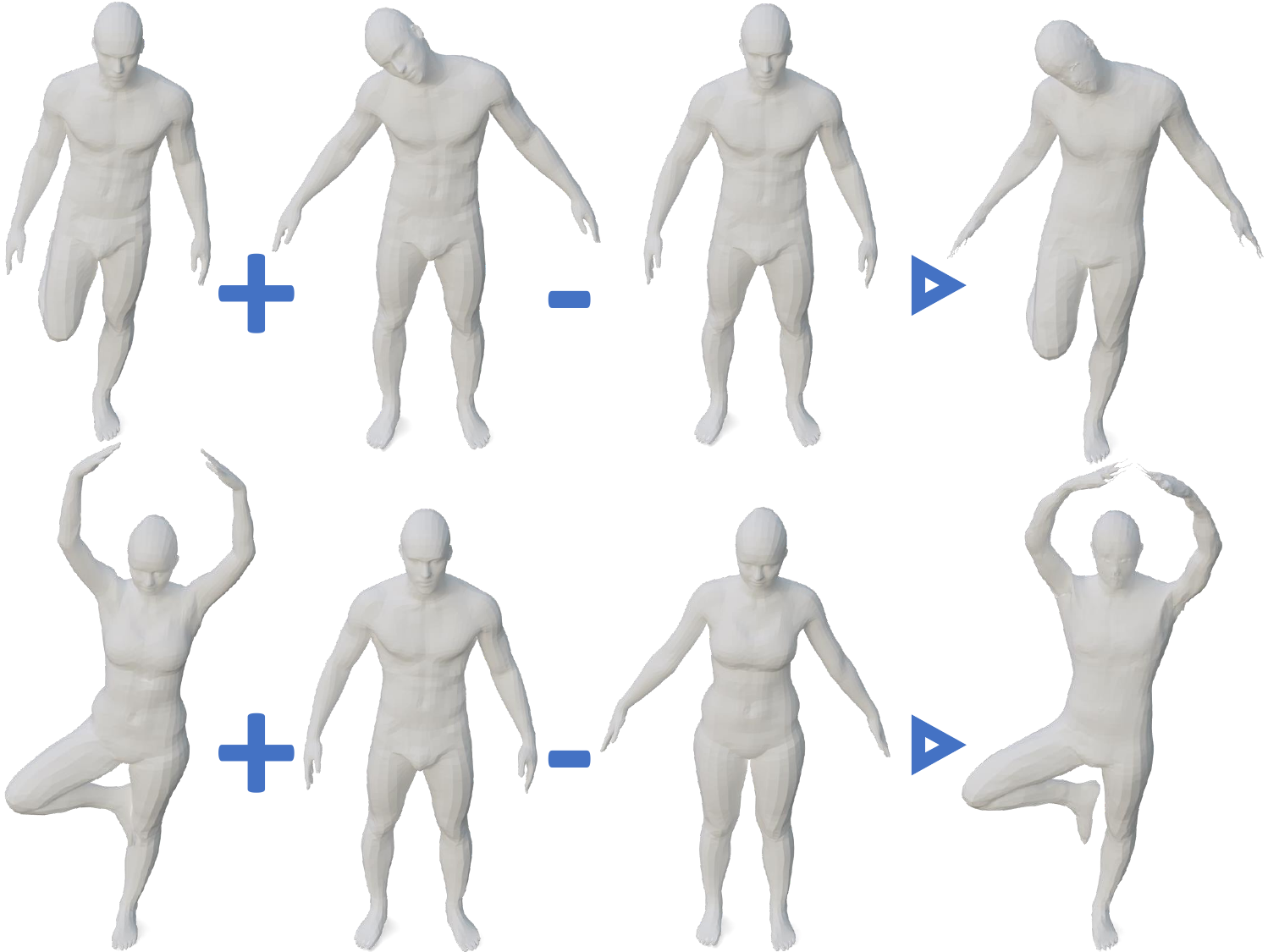}}
    \caption{Combining two positions of \textit{FAUST} test shapes (upper row) and transferring the pose from a female to a male individual.}
    \label{gen_faust2}
\end{center} 
\vspace{-2em}
\end{figure}

\subsubsection{Shape Generation and Manipulation}

To show that the shape features lie on a smooth manifold and that the network is not overfitting to the training samples, we generate new shapes by sampling from the latent feature space from the supervised ``unknown individuals'' setting. We conduct three different generative experiments on the \textit{FAUST} shape collection: interpolation of two test shapes (Figure \ref{gen_faust}), as well as generation of combined positions and feature transfer between two different bodies (Figure \ref{gen_faust2}). The figures show the smooth and well-formed generated shapes with correctly and naturally positioned limbs. While the feature transfer results can be compared to an actual shape from the collection, our interpolation and position combination experiments create well-formed samples that cannot be found in the shape collection. Additionally, the combination of positions and feature transfer shows that our embedding space allows algebraic manipulation (addition and subtraction) of shape embeddings. 

\section{Conclusion \& Limitations}

In this work, we introduce a novel unsupervised method for learning representations of diverse deformable shape collections. Our presented autoencoder architecture reconstructs shapes in higher quality than various baseline methods. 
Additionally, the computed features of meshes with different connectivity and from different categories lie in the same embedding space.
This smooth embedding space, which allows for interpolation and algebraic manipulation, motivates the application of spectral pooling for generative models.

One limitation of our work is that it does not yet handle shape collections with high non-isometry, such as the \textit{GALLOP} shape collection, where we were unable to learn good point-to-point maps between different classes (\ie, between horses and elephants). While our network uses a set of fixed templates for reconstruction, it would be interesting to investigate whether the decoder can generate multiple mesh topologies without the use of a template. 
We leave this as future work.

\paragraph{Acknowledgements}
The authors would like to thank the anonymous reviewers for their valuable suggestions. Parts of this work were supported by the ERC Starting Grant No. 758800 (EXPROTEA), the ANR AI Chair AIGRETTE, and the GlobalMathNetwork from the Hausdorff Center for Mathematics.

 \clearpage
 \newpage
 \appendix

\section{Supplementary Material}

We compile the results and discussions that were not accommodated in the main manuscript due to page constraints. 

Specifically, Section \ref{sec:imp_details} provides details on the implementation aspects of our pipeline. Section \ref{app:truck} elucidates our motivation for unsupervised feature learning on the \textit{TRUCK} dataset and offers additional reconstruction results. Section \ref{sec:ablation} introduces an ablation study concerning the components of our pipeline. Finally, Section \ref{app:pooling} delves into our interpretation of pooling in 2D and surface meshes, and it formulates and provides proof for the Lemma mentioned in Section 4.1 of the main text.

\subsection{Implementation Details}
\label{sec:imp_details}
For our experiments concerning the extraction of point-to-point maps in Section 4.2 of the main text, we use a functional map of size $k=30$. Concerning the feature extractor DiffusionNet~\cite{Sharp2020}, we use the default segmentation configuration provided by the authors\footnote{\url{https://github.com/nmwsharp/diffusion-net}}. After extracting the first set of functional maps, we refine them using ZoomOut \cite{Melzi_2019} using 30 iterations, from $k=30$ to $k=120$. For the Laplace-Beltrami computation, we use the cotangent discretization scheme \cite{Pinkall1993}. 

Concerning our autoencoder architecture in Section 4.3 of the main text, we use the same segmentation configuration of DiffusionNet for both the encoder and decoder. 
When using true point-to-point maps as supervision, we do not apply dropout inside the DiffusionNet blocks.
We chose $F$, the number of features output by the DiffusionNet layer in the encoder, and $k_2$, the dimensions of the CCLB, for all the shape collections in a way, such that the embedding dimension is approximately $k_2 \times F = 1024 $.

In all our experiments, we train our networks using Adam optimizer \cite{kingma2017adam} with an initial learning rate of 0.001.
For the autoencoder training losses in Section 4.4 of the main text, we use $\lambda = 10$. Concerning the reconstruction loss, due to the large size of the matrices $D^S$ and $D^X$, the shapes $X$ and $\Pi S$ are resampled to 20000 vertices if they are larger than it, only during the loss computation.

\subsection{\textit{TRUCK} shape collection and additional reconstruction results}
\label{app:truck}

In a car crash simulation, the different car components are generally represented by surface meshes, which makes our method applicable to this kind of data. 
From simulation run to simulation run, the car model parameters are modified to achieve multiple design goals, e.g. crash safety,
weight, or performance. 
Depending on the chosen model and simulation parameters, the car model often deforms in different patterns. 
Since the simulations nowadays contain detailed information for up to two hundred time steps and more than ten million nodes, their analysis is challenging and generally assisted by dimension reduction methods. 
One goal is the detection of clusters corresponding to different deformation patterns in the components’ embeddings. We visualize our 2D embedding of two components that deform in 32 simulations over time in Figure 5 in the main text. 
This way relations between model parameters and the deformation behavior are discovered more easily and the analysis of car crash simulations is accelerated \cite{Bohn2013,Hahner2020}.

We provide reconstruction results from the supervised experiment of a car component from the \textit{TRUCK} dataset in Figure \ref{fig:recon_truck}, which manifests two different deformation patterns over time, that are visible in the embedding space in Figure 5.

\begin{figure}[t]
\begin{tabular}{C{0.22\linewidth} C{0.2\linewidth} C{0.2\linewidth} C{0.2\linewidth} }
  GT & MASER & CoSMA & Ours\\
\end{tabular}
   {\includegraphics[width=\linewidth, trim=0cm 10cm 0cm 0cm, clip]{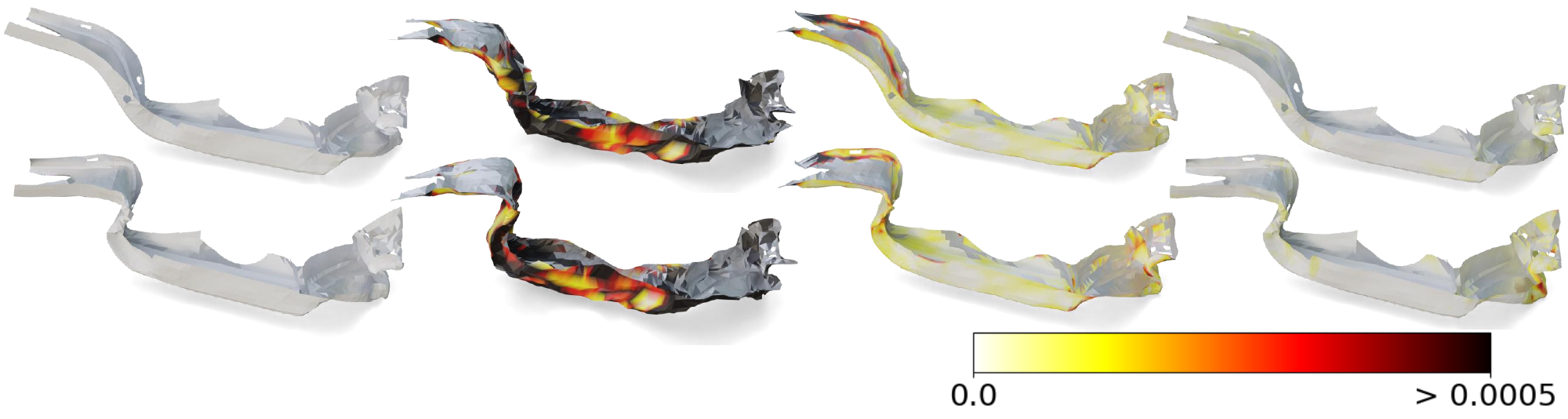}} 
\caption{Reconstructions of a car component, which deforms in two different patterns (first and second row), from the \textit{TRUCK} dataset. Vertex-wise MSE is highlighted.}
\label{fig:recon_truck}
\vspace{-1em}
\end{figure}

\subsection{Ablation Study}
\label{sec:ablation}

To validate our pipeline's components, we performed an ablation study.

First, we wanted to examine the role of spectral pooling and the reconstruction loss. It's important to note that the reconstruction loss can't be used by itself since it's rotation-invariant. For this experiment, we utilized the \textit{FAUST} dataset in an interpolation setting, mirroring Section 5.2.1. We carried out four experiments: the first one employed our complete pipeline with supervised maps; the second used unsupervised maps; the third operated with unsupervised maps but omitted the reconstruction loss; and the fourth involved supervised maps without the spectral pooling (instead, we opted for global pooling as presented in Table 1 of the main text, following the approach of previous studies like \cite{Litany2018}). The outcomes are presented in Table \ref{tab:ablation}. They indicate that each component is crucial for achieving the best results. Notably, the spectral mesh pooling's contribution to the combined embedding space is significant; using just global pooling leads to a marked drop in performance.

\begin{table}[h]  
    \centering
    \ra{1.0}
            \begin{tabular}{@{}lr@{}}

\toprule
Setting & \textit{FAUST} dataset  \\
\midrule
w/o limit shape &  16.7\\
w/o reconstruction loss & 4.3\\
with unsupervised  maps & 2.0 \\
with supervised maps & \textbf{0.7}\\
\bottomrule

            \end{tabular}
\caption{Ablation study on the component of our pipeline.}
\label{tab:ablation}
\end{table}

Secondly, we study the impact of the size of the projection on the limit shape $k_2$. To do this, we use the GALLOP dataset in a supervised setting similar to Section 5.2.1. 
We keep the size of the embedding space fixed (equal to 1024), which is determined by $k_2$, the size of the limit shape multiplied by the feature dimensions $F$ of the encoder. We increase $k_2$ monotonically from $k_2 = 1$ to $k_2 = 70$, while adapting the feature dimension $F$ accordingly. 
The results are summarized in Figure \ref{fig:size_limit}, which shows that the higher the dimension $k_2$ of the limit shape, the better the performance, corresponding to a bigger pooling in the spectral space. However, it can also be seen that performance starts to deteriorate with bigger $k_2$. 
This is explained by the fact that we keep the dimension $k_2 \times F $ of the embedding space fixed, and hence fewer features are extracted with higher $k_2$, which is not sufficient for encoding and high-quality decoding.

\begin{figure}[h]
\centering
   {\includegraphics[width=0.85\linewidth]{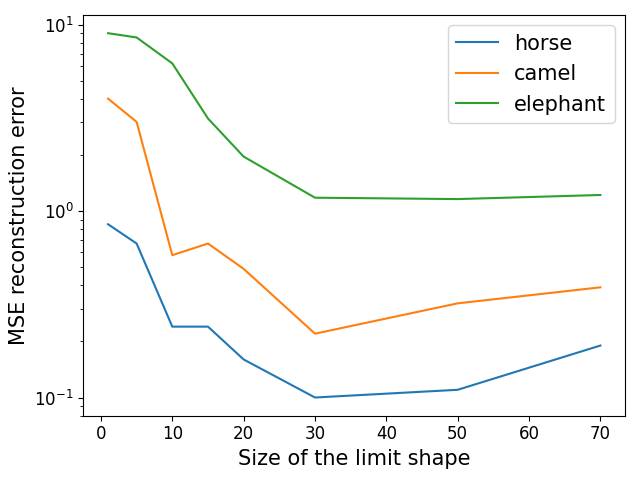}} 
\caption{Impact of $k_2$, the size of the limit shape. If $k_2=1$ this corresponds to global average pooling; see Section \ref{app:pooling}.}
\label{fig:size_limit}
\end{figure}

\subsection{Motivation: Pooling in 2D as a projection to a common basis and Lemma \ref{mean_pooling}} 
\label{app:pooling}

In Section 4.1 of the main text, we introduced a new spectral mesh pooling operator. This operator reduces the dimensionality of the meshes, enabling us to manage meshes with varying connectivity and represent them in a unified low-dimensional embedding space.

In the case of classical representation learning for 2D images with convolutional networks, one has a fixed-size grid and, in fact, all samples are in 1-to-1 correspondence. 
The convolutional filters with stride 1 calculate vertex-wise features, then pooling summarizes many vertex-wise features, going from $n$ pixels to $k$. 
This is done symmetrically for all the images and the features from different samples are comparable to each other because of the 1-to-1 correspondence. 
Let us consider the canonical basis for images of size $n=4 \times 4$, and the pooling size $k=2 \times 2$, in this case:
    \[
    \begin{pmatrix}
        1 & 0\\
        0 & 0
    \end{pmatrix}, 
    \begin{pmatrix}
        0 & 1\\
        0 & 0
    \end{pmatrix},
    \begin{pmatrix}
        0 & 0\\
        1 & 0
    \end{pmatrix},
    \begin{pmatrix}
        0 & 0\\
        0 & 1
    \end{pmatrix}.\]
is a common basis for all the images with $2 \times 2$ pixels.
The projection vector from $n$ pixels towards one of low-dimensional basis has ones in the corresponding corner. For the first common basis, it is 
    \[ \frac{1}{4} \begin{pmatrix}
        1 & 1 & 0 & 0\\
        1 & 1 & 0 & 0\\
        0 & 0 & 0 & 0\\
        0 & 0 & 0 & 0
    \end{pmatrix}, \]
displayed in the shape of the image.
Therefore, pooling in 2D can also be interpreted as a projection from $n$ dimensions to a set of common basis functions in $k$ dimensions. 
This projection reduces the dimensionality of the data while the dimensionality of the pixel-wise features stays the same.
Because of the 1-to-1 correspondence, all the images are described in the same basis. 
A similar pooling operator cannot be constructed for meshes with different mesh connectivities. 
We can only obtain point-to-point maps between the shapes that allow the projection of a function from one shape to another.

To solve the pooling for meshes, we propose to adapt the CCLB method (initially developed for deformation detection) and introduce a novel intrinsic spectral mesh pooling.
We project vertex-wise features that are calculated for every shape separately to the common CCLB basis, reducing the dimension from the number of vertices to the size of the limit shape. 
We calculate the limit shape basis CCLB as described in section 3.4. 
It has dimension $k_2$ and uses eigendecompositions of the Laplacians of size $k_1 \geq k_2$. It will be the common basis for the low-dimensional embedding space.
For the spectral unpooling, we project the features from the limit shape basis back to the vertex representation.

If the dimensionality of the CCLB is 1 ($k_1=k_2=1$), the projection of the shape features into the CCLB corresponds to a global $\pm$~mean pooling for all the shapes in the collection. 
Furthermore, the inverse of this operation duplicates the average feature into the shape's vertices, similar to upsampling in the 2D case. 
Also, the sign of the resulting global mean pooling function from all shapes in the shape collection to the CCLB is the same, which makes the different low-dimensional representations comparable to each other.
We formally state this observation in the following lemma.

\begin{lemma}
\label{mean_pooling}
If $k_1=k_2=1$, there are only two possible solutions for the projection $\tilde{Y}_i^{\dagger} \Phi_i^{\dagger} $ from the vertex-wise features to the CCLB for all shapes $S_i, i=1,2,\dots$ and the projection from the CCLB to a template shape $\Phi_t \tilde{Y}_t$. Either
\begin{equation}
    \tilde{Y}_i^{\dagger}  \Phi_i^{\dagger}  x_i = \mathrm{mean}(x_i) \hspace{1ex} \forall i \textrm{ and } \Phi_t \tilde{Y}_t = \mathbf{1}_{n_t} 
\end{equation}
or
\begin{equation}
\tilde{Y}_i^{\dagger}  \Phi_i^{\dagger} x_i = -\mathrm{mean}(x_i) \hspace{1ex} \forall i \textrm{ and } \Phi_t \tilde{Y}_t = -\mathbf{1}_{n_t} 
\end{equation}
 with $\mathrm{mean}:\R^{n_i\times d} \to \R^{d} $ is the vertex-wise average function, and $\mathbf{1}_{n_t}$ is  the column-vector with only ones in $\R^{n_t}$, and $n_t$ being the number of vertices of the template shape.
\end{lemma}

\begin{proof}
At first, we proof that $\tilde{Y}_i^{\dagger} \Phi_i^{\dagger} x_i = \pm \mathrm{mean}(x_i) $ for a fixed $i$.
If $k_1=k_2=1$ we have 
\begin{equation}
\Phi_i = \pm \textbf{1}_{n_i}
\end{equation}
being the eigenvector corresponding to the smallest eigenvalue $\Lambda_i = 0$, because the sum of all values in each row of the Laplacian $\La_i$ is 1.
Therefore,
\begin{equation}
\label{phiinv}
\Phi_i^{\dagger} x_i = \frac{1}{n_i} \Phi_i^T x_i = \pm \mathrm{mean}(x_i).
\end{equation}
The functional map 
\begin{equation}
\label{CC}
    C_{ij} = \Phi_j^T \Phi_i \in \R^{1 \times 1}
\end{equation}
is 1 or -1, projecting only constant functions from shape $j$ to shape $i$. It holds $C_{ij} = C_{ji}$.
If $k_1=1$, the optimization problem to compute the Consistent Latent Basis (CLB)
\begin{equation}
\label{CLB}
    \min_{Y} \norm{C_{ij} Y_i  \text{ - }  Y_j } \textrm{ s.t. } \sum_i Y^T_i  Y_i = I
\end{equation}
has the solutions:
\begin{equation} 
\begin{split}
\textrm{if } C_{ij} = C_{ji} = 1 & \Rightarrow Y_i = Y_j \in \{-1,1\}\\
\textrm{else } C_{ij} = C_{ji} = -1 & \Rightarrow Y_i = - Y_j \in \{-1,1\}.
\end{split}
\end{equation}
Since $\Lambda_i = 0$, the matrix E in algorithm 1 from \cite{Huang2019} is 0. Therefore, the possible solutions for its eigenvector $U$ are -1 and 1. For the calculation of the CCLB, this leads to
\begin{equation}  
\tilde{Y}_i = Y_i U \in \{1, -1\}.
\end{equation}
For the inverse it holds
\begin{equation} 
\label{Yinv}
\tilde{Y}_i^{\dagger} = \tilde{Y}_i. 
\end{equation}
From (\ref{phiinv}) and (\ref{Yinv}) follows
\begin{equation} 
\tilde{Y}_i^{\dagger} \Phi_i^{\dagger} x_i = \pm \mathrm{mean}(x_i)
\end{equation}
and all entries of the matrix have the same sign.

In a second step, we prove by contradiction that the non-zero entries of the matrix products $\tilde{Y}_i^{\dagger} \Phi_i^{\dagger}$ have the same sign for all $i=1,2,\dots$ .\\
Assume that the sign of $\tilde{Y}_i^{\dagger} \Phi_i^{\dagger}$ is different from the sign of $\tilde{Y}_j^{\dagger} \Phi_j^{\dagger}$ for $i \neq j$. 
Without loss of generality, assume the sign of $\tilde{Y}_i^{\dagger} \Phi_i^{\dagger}$ to be positive. 
Therefore, the sign of $\tilde{Y}_i^{\dagger}$ is the same as the sign of $\Phi_i^{\dagger}$.
Then, either $\tilde{Y}_j^{\dagger}$ or $\Phi_j^{\dagger}$ has a different sign.\\
If $\tilde{Y}_j^{\dagger} = - \tilde{Y}_i^{\dagger}$,  then ${Y}_j  = - {Y}_i$ and therefore $C_{ij} = C_{ji} = -1$ because $Y_i$ and $Y_j$ solve (\ref{CLB}). 
From (\ref{CC}) follows that $\Phi_j$ and $\Phi_i$ have different signs, which is a contradiction to $\Phi_i^{\dagger}$ having the same sign as $\Phi_j^{\dagger}$.\\
If in the other case $\Phi_i^{\dagger}$ has a different sign than $\Phi_j^{\dagger}$, $C_{ij} = C_{ji} = -1$ because of (\ref{CC}). It follows $Y_i = - Y_j$, which is a contradiction to $\tilde{Y}_j^{\dagger}$ having the same sign as $\tilde{Y}_i^{\dagger}$.

Finally, the entries of the matrix product $\Phi_t \tilde{Y}_t $, which projects the features from the CCLB representation to the template shape,
have the same sign as $\tilde{Y}_i^{\dagger} \Phi_i^{\dagger}$.

\end{proof}

 {\small
\bibliographystyle{ieeenat_fullname}
\bibliography{references}
}

\end{document}